\documentclass[conference]{IEEEtran}
\usepackage[english]{babel}

\usepackage{stfloats}
\usepackage{type1cm} 
\usepackage{graphicx} 
\usepackage{xspace} 
\usepackage{balance} 
\usepackage{booktabs} 
\usepackage{multirow} 
\usepackage[font={bf}, tableposition=top]{caption} 
\usepackage{subcaption} 
\usepackage{bold-extra} 
\usepackage{algorithm}
\usepackage{algorithmic}
\usepackage{microtype} 
\usepackage{siunitx} 
\usepackage{xfrac} 
\usepackage{amsmath}
\usepackage{amssymb}
\usepackage{mathtools}
\usepackage{amsfonts}
\usepackage{amsthm}
\usepackage{bbm}
\usepackage[hyphens]{url} 
\usepackage[bookmarks, pdftex, colorlinks=true]{hyperref} 
\usepackage{cleveref} 
\usepackage{csquotes}
\usepackage{bm}
\usepackage{xspace}
\usepackage{physics}
\usepackage{newunicodechar}
\newunicodechar{⟓}{\ensuremath{\uplus}}
\usepackage{tikz}
\usetikzlibrary{positioning}
\usetikzlibrary{bayesnet}
\usetikzlibrary{arrows}
\usetikzlibrary{shapes}
\usetikzlibrary{fit}
\usepackage{tikz-cd}
\usepackage{bbold}
\usetikzlibrary{calc, positioning, arrows.meta, shapes.geometric, fit, backgrounds}
\usepackage{pgf}
\usepackage{tikz-3dplot}
\usepackage{tcolorbox}
\usepackage{pgfplots}
\pgfplotsset{compat=1.18}

\newcommand{\spara}[1]{\smallskip\noindent\textbf{#1}}

\usepackage[dvipsnames]{xcolor}
\hypersetup{
   bookmarks, pdftex,
   colorlinks=true,
   pagebackref=true, backref=page,
   linkcolor={red!50!black},
   filecolor={green!50!black},
   citecolor={green!50!black}, 
   urlcolor={blue!80!black},
}

\graphicspath{{fig}}

\theoremstyle{plain}
\newtheorem{theorem}{Theorem}

\crefname{theorem}{Thm.}{Thms.}
\crefname{proposition}{Prop.}{Props.}
\crefname{lemma}{lem.}{lems.}
\crefname{corollary}{Cor.}{Cors.}
\crefname{definition}{Def.}{Defs.}
\crefname{section}{Sec.}{Secs.}
\crefname{figure}{Fig.}{Figs.}
\crefname{problem}{Prob.}{Probs.}
\crefname{appendix}{App.}{Apps.}
\crefname{equation}{Eq.}{Eqs.}
\crefname{table}{Tab.}{Tabs.}

\newcommand{\reall}{\ensuremath{\mathbb{R}}\xspace}

\definecolor{mypurple}{RGB}{254, 68, 218}

\IEEEoverridecommandlockouts

\def\BibTeX{{\rm B\kern-.05em{\sc i\kern-.025em b}\kern-.08em
    T\kern-.1667em\lower.7ex\hbox{E}\kern-.125emX}}

\title{Colored Markov Random Fields for
Probabilistic \\ Topological Modeling\\
\thanks{The work was supported by the SNS JU project 6G-GOALS   
under the EU’s Horizon program Grant Agreement No 101139232, by Huawei Technology France SASU under Grant N. Tg20250616041, and by the  European Union under the Italian National Recovery and Resilience Plan of NextGenerationEU, partnership on Telecommunications of the Future (PE00000001- program RESTART).}
}

\author{
\IEEEauthorblockN{
Lorenzo Marinucci\IEEEauthorrefmark{1},
Leonardo Di Nino\IEEEauthorrefmark{2},
Gabriele D'Acunto\IEEEauthorrefmark{2},\\
Mario Edoardo Pandolfo\IEEEauthorrefmark{3},
Paolo Di Lorenzo\IEEEauthorrefmark{2},
Sergio Barbarossa\IEEEauthorrefmark{2}
}

\IEEEauthorblockA{
\IEEEauthorrefmark{1}Statistical Sciences Dept., Sapienza University of Rome, Italy \;\\
\IEEEauthorrefmark{2}DIET  Dept., Sapienza University of Rome, Italy \;\\
\IEEEauthorrefmark{3}DIAG Dept., Sapienza University of Rome, Italy
}

\IEEEauthorblockA{
\{l.marinucci,
leonardo.dinino,
gabriele.dacunto,
marioedoardo.pandolfo,
paolo.dilorenzo,
sergio.barbarossa\}@uniroma1.it
}
}

\begin{document}
\maketitle

\begin{abstract}
Probabilistic Graphical Models (PGMs) encode conditional dependencies among random variables using a graph--nodes for variables, links for dependencies--and factorize the joint distribution into lower-dimensional components.
This makes PGMs well-suited for analyzing complex systems and supporting decision-making. Recent advances in topological signal processing highlight the importance of variables defined on topological spaces in several application domains. 
In such cases, the underlying topology shapes statistical relationships, limiting the expressiveness of canonical PGMs. 
To overcome this limitation, we introduce Colored Markov Random Fields (CMRFs), which model both conditional and marginal dependencies among Gaussian edge variables on topological spaces, with a theoretical foundation in Hodge theory. CMRFs extend classical Gaussian Markov Random Fields by 
including link coloring: connectivity encodes conditional independence, while color encodes marginal independence.
We quantify the benefits of CMRFs through a distributed estimation case study over a physical network, comparing it with baselines with different levels of topological prior.
\end{abstract}

\section{Introduction}
Probabilistic Graphical Models (PGMs) provide a unified framework for representation,
inference, and learning in complex systems with structured randomness, described through
dependency relations among random variables \cite{scutari,koller2009probabilistic}. In
PGMs, each variable is associated with a node of a graph, and links indicate where
\emph{statistical interactions} between variables occur. A well-studied instance is the
Markov Random Field (MRF), which uses an undirected graph to encode conditional
independence statements within a collection of variables \cite{li2009markov}. In
particular, a Gaussian MRF specializes this model to multivariate Gaussian
distributions, with graph connectivity captured by the sparsity pattern of the
precision matrix \cite{rue2005gaussian}.

In many application domains, however, the random variables of interest do not arise as abstract
labels on an unstructured index set, but are inherently tied to an underlying physical or relational network. Furthermore, variables are often not confined to the nodes of the network, but can also be associated with groups of two or more elements, composing so-called \textit{higher-order signals} \cite{giusti2016twos,lambiotte2019higherorder}. Examples include flows in water or power distribution infrastructures, traffic load on communication networks or multi-way interactions between individuals in social networks \cite{cattai2025leak,kazim2025edge, frantzen2025hlsad}.
In this setting, Topological Signal Processing (TSP) offers a principled framework for
analyzing signals of varying order \cite{barbarossa2020topological,schaub2021signal}. TSP represents the underlying physical domain by discrete topological spaces such as simplicial or cell
complexes, which extend graphs by adding triangles and higher-order components. In addition, over simplicial complexes, TSP relates signals across different orders by means of discrete Hodge theory, capturing how interactions between vertex-, edge-, and triangle-level signals are structured.

While most contributions in TSP focus on deterministic signals, recent works have proposed Gaussian models for signals supported on discrete topological spaces \cite{sardellitti2022probabilistic, yang2024hodgeedgegp,alain2024gaussian,gurugubelli2024simplicial,marinucci2025simplicialgaussianmodelsrepresentation}. In detail, \cite{sardellitti2022probabilistic} introduces a probabilistic topological model on edge signals, and a lasso algorithm to estimate their conditional dependencies and the domain's topology. Extending the graph-based approach, works in \cite{yang2024hodgeedgegp,alain2024gaussian} propose discrete Matérn kernels to model Gaussian processes over simplicial and cell complexes. Authors in \cite{gurugubelli2024simplicial} formulate a probabilistic model for edge flows on simplicial complexes, together with a topology inference algorithm. Finally, in our previous work \cite{marinucci2025simplicialgaussianmodelsrepresentation}, we build on Hodge theory to introduce the Simplicial Gaussian Model (SGM), a class of random signals supported on all levels of a simplicial complex, deriving a marginal model for edge signals. We also propose a maximum-likelihood method to estimate the parameters of the SGM and the higher-order topology from edge observations.

However, to the best of our knowledge, the probabilistic independence structure of Gaussian edge signals on simplicial complexes is still underexplored. 
Building on \cite{marinucci2025simplicialgaussianmodelsrepresentation}, we introduce a class of graphical models for representing such signals, named Colored Markov Random Fields (CMRFs). 
The CMRF not only encodes 
 \textit{where} statistical interactions between edge variables occur, but also captures interaction \textit{types} by link coloring. This distinction reflects the natural hierarchy characterizing simplicial complexes, on which the edge variables are defined. We prove that our approach preserves the expressiveness of MRFs in terms of conditional independence statements, while also providing a layer of marginal independence relations, induced by the heterogeneity of interactions over the topological domain. Finally, we validate this theoretical framework by addressing a distributed estimation task over the links of a physical network. Compared with different baselines providing topological information, numerical results show that the CMRF consistently yields better estimation performance.
\section{Background}\label{sec:background}

\spara{Simplicial Complexes.} A \textit{simplicial complex} $\mathcal{X}$ is a set $\mathcal{V}$ together with a family of subsets (\textit{simplices}) $\mathcal{S}$ such that: (i) for every $v \in V$, $\{v\}$ belongs to $\mathcal{S}$, (ii) if $\sigma_i \in \mathcal{S}$ and $\sigma_j \subset \sigma_i$, then $\sigma_j \in \mathcal{S}$. The \textit{order} $k$ of a simplex (also referred to as \textit{$k$-simplex}) is given by its cardinality minus one. 
The largest order $K$ of the simplices gives the \textit{dimension} of the complex. For instance, a $2-$dimensional simplicial complex include the set $\mathcal{V}$ of $0-$simplices, the set $\mathcal{E}$ of $1$-simplices and the set $\mathcal{T}$ of $2$-simplices, denoted \textit{vertices, edges} and \textit{triangles}, respectively.

\spara{Incidence matrices.} 
Each simplex in $\mathcal{X}$ can be assigned one of two possible orientations. 
Let $\mathcal{X}$ be a fully-oriented, $K-$dimensional simplicial complex. Its incidence relations are encoded by a set of \textit{incidence matrices} $\{\mathbf{B}_k\}_{k=1}^K$ (see \cite{grady2010discrete}).  
For each $k=1,...,K$, the entries of the matrix $\mathbf{B}_k$ are
\[
\mathbf{B}_k(i,j)=
\begin{cases}
  +1\ , & \text{if $\sigma_j^{k-1} \subset \sigma_i^{k}$ with coherent orientation},\\[2pt]
  -1\ , & \text{if $\sigma_j^{k-1} \subset \sigma_i^{k}$ with opposite orientation},\\[2pt]
   0\ , & \text{otherwise},
\end{cases}
\]
with $\sigma_j^{k-1}$ and $\sigma_i^{k}$ simplices of order $k-1$ and $k$, respectively. For each $k=1,...,K-1$, the incidence matrices satisfy the \textit{chain complex property}, i.e. $\mathbf{B}_k \mathbf{B}_{k+1} = \mathbf{0}$.

\spara{Hodge Laplacians.} Graph Laplacians admit generalization to simplicial complexes. Given the incidence matrices $\{\mathbf{B}_k\}_{k=1}^K$, the \textit{combinatorial Laplacians} are defined as:
\begin{align}
 &\mathbf{L}_0 =\mathbf{B}_1\mathbf{B}^T_1, \notag\\
 &\mathbf{L}_k = \mathbf{B}^T_k\mathbf{B}_k+\mathbf{B}_{k+1}\mathbf{B}^T_{k+1}, \quad k=1,...,K-1 \\
 &\mathbf{L}_K = \mathbf{B}^T_K\mathbf{B}_K .\notag
\end{align}
For each intermediate order $k=1,...,K-1$, they are composed of two terms. The \textit{lower Laplacian} $\mathbf{L}_{k,d} = \mathbf{B}_k^\top \mathbf{B}_k$ describes the \textit{lower-adiacency} relations between $k-$simplices, i.e. if they share a common $k-1$ face. The \textit{upper Laplacian} $\mathbf{L}_{k,u} = \mathbf{B}_{k+1}\mathbf{B}_{k+1}^\top$ describes the \textit{upper-adiacency} relations between $k-$simplices, i.e. if they are face of a common $k+1$ simplex.    

\spara{Topological Signals.} 
From now on, we focus on simplicial complexes of dimension two. A \emph{topological signal} on a $2$-dimensional simplicial complex is a real-valued function defined on its simplices of a fixed order. In particular, we consider $0$-, $1$-, and $2$-signals defined on the sets of vertices $\mathcal{V}$, edges $\mathcal{E}$, and triangles $\mathcal{T}$, respectively:
\[
\mathbf{x}_\mathcal{V}:\mathcal{V}\to\mathbb{R},\qquad
\mathbf{x}_\mathcal{E}:\mathcal{E}\to\mathbb{R},\qquad
\mathbf{x}_\mathcal{T}:\mathcal{T}\to\mathbb{R}.
\]
Assigning an ordering to the simplices, topological signals can be defined as vectors $\mathbf{x}_\mathcal{V} \in \mathbb{R}^{|\mathcal{V}|}$, $\mathbf{x}_\mathcal{E} \in \mathbb{R}^{|\mathcal{E}|}$, and $\mathbf{x}_\mathcal{T} \in \mathbb{R}^{|\mathcal{T}|}$. 

\spara{Hodge Decomposition.} The property $\mathbf{B}_k \mathbf{B}_{k+1} = \mathbf{0}$ induces orthogonal decompositions of the signal spaces, referred to as \textit{Hodge decompositions}. For the edge signal spaces (order $k=1$), this consists of:
\begin{equation}
\mathbb{R}^{|\mathcal{E}|} = 
\mathrm{im}(\mathbf{B}_1^\top)
\oplus
\mathrm{im}(\mathbf{B}_2)
\oplus
\mathrm{ker}(\mathbf{L}_1)\ . \label{eq:Hodge_dec} 
\end{equation}
Consequently, each edge signal $\mathbf{x}_\mathcal{E} \in \mathbb{R}^{|\mathcal{E}|}$ can be written as  
\begin{equation}\label{eq:Hodge_dec_2}
\mathbf{x}_\mathcal{E} = \mathbf{B}_1^\top \mathbf{x}_\mathcal{V} + \mathbf{B}_2 \mathbf{x}_\mathcal{T} + \mathbf{h}_\mathcal{E}\ ,
\end{equation}
with $\mathbf{x}_\mathcal{V} \in \mathbb{R}^{|\mathcal{V}|}$, $\mathbf{x}_\mathcal{T} \in \mathbb{R}^{|\mathcal{T}|}$ and $\mathbf{h}_\mathcal{E} \in \mathbb{R}^{|\mathcal{E}|}$. By analogy with Hodge theory on Riemannian manifolds, the three terms in~\Cref{eq:Hodge_dec_2} can be interpreted as the \emph{irrotational}, \emph{solenoidal}, and \emph{harmonic} components of the edge signal $\mathbf{x}_\mathcal{E}$, respectively~\cite{barbarossa2020topological}.

\spara{Gaussian Markov Random Fields.} 
Let $\mathbf{x}_\mathcal{N}
\sim N(0,\mathbf{\Sigma})$ be a Gaussian random vector, with covariance $\mathbf{\Sigma}\succ0$. We denote $\bm{\Omega} = \mathbf{\Sigma}^{-1}$ the \textit{precision matrix} of $\mathbf{x}_\mathcal{N}$. A \textit{Gaussian Markov Random Field} (GMRF) represents $\mathbf{x}_\mathcal{N}$ by an undirected graph \(G = (\mathcal{N}, \mathcal{L})\): its nodes correspond to the index set $\mathcal{N}$, while its links are encoded by the precision matrix, that is,
\begin{equation}
 (n_i,n_j) \in \mathcal{L}
\;\Longleftrightarrow\;
\bm{\Omega}_{ij} \neq 0, \qquad i \neq j.   
\end{equation}
Accordingly, for $i \neq j$, if \(\bm{\Omega}_{ij} \neq 0\) we say that variables \(x_{n_i}\) and \(x_{n_j}\) \textit{interact directly}, in the sense that their dependence is expressed by a link of the corresponding GMRF.

\spara{Global Markov Property.} Let \(G = (\mathcal{N}, \mathcal{L})\) be a GMRF associated to $\mathbf{x}_\mathcal{N}$. For disjoint $\mathcal{N}_\mathcal{A},\mathcal{N}_\mathcal{B},\mathcal{N}_\mathcal{S}\subseteq\mathcal{N}$, we indicate with $\mathcal{N}_\mathcal{A} \perp_G \mathcal{N}_\mathcal{B} \ | \ \mathcal{N}_\mathcal{S}$ the fact that every path from $\mathcal{N}_\mathcal{A}$ to $\mathcal{N}_\mathcal{B}$ in $G$ meets $\mathcal{N}_\mathcal{S}$. 
Then, $\mathbf{x}_\mathcal{N}$ satisfies the \textit{global Markov property} in $G$:
\[
 \mathcal{N}_\mathcal{A} \perp_G \mathcal{N}_\mathcal{B} \ | \ \mathcal{N}_\mathcal{S}
\;\Longrightarrow\;
\mathbf{x}_{\mathcal{N}_\mathcal{A}}\;\perp\!\!\!\perp\;\mathbf{x}_{\mathcal{N}_\mathcal{B}}\;\big|\;\mathbf{x}_{\mathcal{N}_\mathcal{S}}.
\]
Therefore, through this property, GMRFs allow to elaborate probabilistic statements of conditional independence between jointly Gaussian variables.
\

\section{Colored Markov Random Fields}
In this section, we introduce CMRFs for representing the independence structure of Gaussian edge signals over simplicial complexes. 
We build on the SGM introduced in~\cite{marinucci2025simplicialgaussianmodelsrepresentation}.

\spara{SGM on edge signals.} 
Let $\mathcal{X}$ be a 2-dimensional simplicial complex, with trivial first homology, 
i.e., $\ker(\mathbf{L}_1) = \{0\}$. We focus on a random edge signal  $\mathbf{x}_\mathcal{E} \in \mathbb{R}^{|\mathcal{E}|}$, that is, a set of random variables indexed by the edges of $\mathcal{X}$. We assume that $\mathbf{x}_\mathcal{E}$ follows a Gaussian distribution $\mathbf{x}_\mathcal{E} \sim N\bigl(0,\bm{\Omega}_\mathcal{E}^{-1}\bigr)$, with precision matrix $\bm{\Omega}_\mathcal{E}\succ0$, and it satisfies
\begin{equation}\label{eq:edge_generation}
    \mathbf{x}_\mathcal{E} 
    \;=\; k^{-1}\mathbf{B}_1^\top \mathbf{x}_\mathcal{V} 
    \;+\; k^{-1}\mathbf{B}_2 \mathbf{x}_\mathcal{T} 
    \;+\; \mathbf{z}_\mathcal{E}, 
\end{equation}
where $k > 0$ is a scalar parameter, $\mathbf{x}_\mathcal{V}$, $\mathbf{x}_\mathcal{T}$ are \textit{latent} random vertex and triangle signals, respectively, and $\mathbf{z}_\mathcal{E} \sim N\bigl(0, k^{-1}\mathbf{I}_{|\mathcal{E}|}\bigr)$ is white Gaussian noise, independent of $\mathbf{x}_\mathcal{V}$ and $\mathbf{x}_\mathcal{T}$. Note that \Cref{eq:edge_generation} expresses a multi-order decomposition of the random edge signal $\mathbf{x}_\mathcal{E}$. In fact, $\mathbf{x}_\mathcal{E}$ incorporates a vertex-generated and a triangle-generated component, plus independent stochastic fluctuations over the edges. Since $\ker(\mathbf{L}_1) = \{0\}$, this decomposition does not include the harmonic component of the Hodge decomposition in \Cref{eq:Hodge_dec_2}, whose incorporation is planned as future work. 

In the SGM, the joint distribution of $(\mathbf{x}_\mathcal{V}, \mathbf{x}_\mathcal{E}, \mathbf{x}_\mathcal{T})$ is Gaussian and its precision matrix $\bm{\Omega}$ is shaped by the simplicial complex's topology (see Eq. (10) in~\cite{marinucci2025simplicialgaussianmodelsrepresentation} for its precise expression). By marginalizing the joint Gaussian model to the edge level via a Schur complement, 
one obtains an explicit expression for $\bm{\Omega}_\mathcal{E}$, 
as stated in Prop.~1 in~\cite{marinucci2025simplicialgaussianmodelsrepresentation}:
\begin{equation}
\label{eq:precision_E_paper}
 \bm{\Omega}_\mathcal{E} 
 \;=\; k\,\mathbf{I}_{|\mathcal{E}|} 
 \;-\; \mathbf{B}_1^\top \operatorname{diag}(\mathbf{d}_\mathcal{V}) \mathbf{B}_1 
 \;-\; \mathbf{B}_2 \operatorname{diag}(\mathbf{d}_\mathcal{T}) \mathbf{B}_2^\top\,,
\end{equation}
where $\mathbf{d}_\mathcal{V} \in \mathbb{R}^{|\mathcal{V}|}$ and 
$\mathbf{d}_\mathcal{T} \in \mathbb{R}^{|\mathcal{T}|}$ are non-negative coefficient vectors and $k$ is sufficiently large to guarantee $\bm{\Omega}_\mathcal{E} \succ 0$.
The entries of $\mathbf{d}_\mathcal{V}$ and $\mathbf{d}_\mathcal{T}$ reflect the variability of 
the latent vectors $\mathbf{x}_\mathcal{V}$ and $\mathbf{x}_\mathcal{T}$, respectively, with zero entries corresponding to the absence of a latent component on the associated vertex or triangle. The two structured terms in~\Cref{eq:precision_E_paper} encode how this variability propagates to the edge variables. 
The terms 
$-\mathbf{B}_1^{\top}\operatorname{diag}(\mathbf{d}_\mathcal{V})\mathbf{B}_1$
and 
$-\mathbf{B}_2\operatorname{diag}(\mathbf{d}_\mathcal{T})\mathbf{B}_2^{\top}$
generate dependencies between edges incident to the same vertex and between edges that
are co-faces of a common triangle, respectively, whenever latent variables are placed
on those simplices. We refer to these as \emph{lower} and \emph{upper} interactions
between edge variables.

\begin{figure}[!t]
    \centering
    \includegraphics[width = \columnwidth]{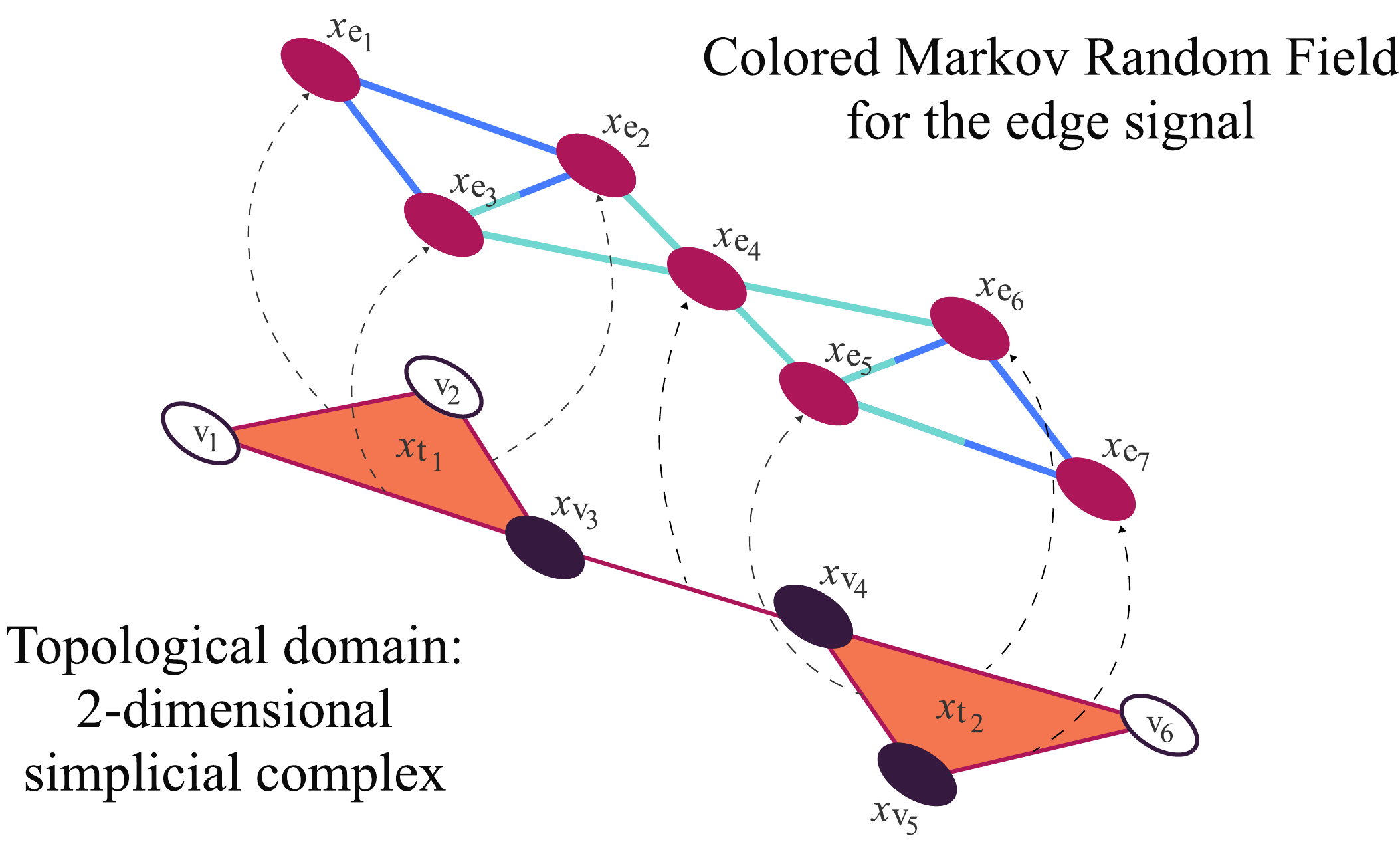}
\caption{$\mathbf{x}_\mathcal{E}:=[x_{e_1},...,x_{e_7}]$ is an edge signal over a $2-$simplicial complex. The CMRF associated to $\mathbf{x}_\mathcal{E}$ distinguishes lower (sky-blue) and upper (dark-blue) interactions.} 
    \label{fig:top-and-cmrf}
\end{figure}

Building on this framework, we define a link-colored graphical representation of $\mathbf{x}_\mathcal{E}$  as follows.

\spara{Color Markov Random Fields.} Let $\mathbf{x}_\mathcal{E} \sim N\bigl(0,\bm{\Omega}_\mathcal{E}^{-1}\bigr)$ be a Gaussian edge signal with precision matrix $\bm{\Omega}_\mathcal{E}\succ0$ as in \Cref{eq:precision_E_paper}. A \textit{Colored Markov Random Field} (CMRF) associated to $\mathbf{x}_\mathcal{E}$ is a link-colored graph $G=(\mathcal{E}, \mathcal{L}_d \cup \mathcal{L}_u)$, with 
   \begin{align}
      \{e_i,e_j\}\in \mathcal L_d 
    &\iff (\mathbf B_1^\top \operatorname{diag}(\mathbf{d}_\mathcal{V}) \mathbf B_1)_{ij}\neq 0 \ , \label{eq:lower_inter} \\
    \{e_i,e_j\}\in\mathcal L_u 
           & \iff (\mathbf B_2 \operatorname{diag}(\mathbf{d}_\mathcal{T}) \mathbf B_2^\top)_{ij}\neq 0 \ . \label{eq:upper_inter}
   \end{align}
In CMRFs, color distinguishes if a link encodes a lower interaction, belonging to $\mathcal L_d$, or an upper interaction, belonging to $\mathcal L_u$. If a link is in both, it carries both colors.

To clarify the construction, 
\Cref{fig:top-and-cmrf} depicts an example. 
Under the current assumptions, we consider a Gaussian edge
signal $\mathbf{x}_\mathcal{E} = [x_{e_1},x_{e_2},\dots,x_{e_7}]$ supported on the 2-simplicial complex in figure. In this case, latent variables are placed on vertices $\{
v_3,v_4,v_5\}$ of the
complex, as well as on both of its triangles, inducing lower and upper interactions between edge variables, respectively. According to the expression of the precision matrix in~\Cref{eq:precision_E_paper}, this configuration implies that the coefficient vector $\mathbf{d}_\mathcal{V}=[d_{v_1},d_{v_2},...,d_{v_6}]$ has $d_{v_1},d_{v_2},d_{v_6}=0$. The resulting CMRF associated to $\mathbf{x}_\mathcal{E}$ is depicted on the top of Fig.~\ref{fig:top-and-cmrf}.

We now state the main result of this work, showing that in a CMRF both the graph
connectivity and link colors give rise to probabilistic indipendence statements for
the variables of interest.  
We start by providing a preliminary definition.

\spara{Color separation.} Let $G$ be a CMRF associated to a random edge signal $\mathbf{x}_\mathcal{E}$. For disjoint $\mathcal{E}_\mathcal{A},\mathcal{E}_\mathcal{B}\subseteq \mathcal{E}$, we say that $\mathcal{E}_\mathcal{A}$ and $\mathcal{E}_\mathcal{B}$ are \textit{separated by color} (and write $\mathcal{E}_\mathcal{A} \perp_{\mathrm{col}} \mathcal{E}_\mathcal{B}$) if there is no \textit{monochromatic} path from $\mathcal{E}_\mathcal{A}$ to $\mathcal{E}_\mathcal{B}$ in $G$.\\
Equivalently, two groups of graph nodes are separated by color if every path connecting them
concatenates both lower and upper interactions. 
\begin{theorem}
 Let $\mathcal{X}$ be a 2-dimensional simplicial complex, with edge set $\mathcal{E}$, $\mathbf{x}_\mathcal{E} \sim N\bigl(0,\bm{\Omega}_\mathcal{E}^{-1}\bigr)$ be a Gaussian edge signal with precision matrix $\bm{\Omega}_\mathcal{E}\succ0$ as in (\ref{eq:precision_E_paper}), and $G$ be the CMRF associated to $\mathbf{x}_\mathcal{E}$. Then, the following hold:
    \begin{enumerate}
        \item $\mathbf{x}_\mathcal{E}$ satisfies the global Markov property in $G$, i.e. given disjoint $\mathcal{E}_\mathcal{A}, \mathcal{E}_\mathcal{B}, \mathcal{E}_\mathcal{S} \subseteq \mathcal{E}$: \[\mathcal{E}_\mathcal{A} \perp_{G} \mathcal{E}_\mathcal{B} \ | \ \mathcal{E}_\mathcal{S}\;\Longrightarrow\; \mathbf{x}_{\mathcal{E}_\mathcal{A}}\;\perp\!\!\!\perp\;\mathbf{x}_{\mathcal{E}_\mathcal{B}}\;\big|\;\mathbf{x}_{\mathcal{E}_\mathcal{S}}\ . \]
        \item Given disjoint $\mathcal{E}_\mathcal{A},\mathcal{E}_\mathcal{B}\subseteq \mathcal{E}$:
        \[\mathcal{E}_\mathcal{A} \perp_{\mathrm{col}} \mathcal{E}_\mathcal{B}\;\Longrightarrow\; \mathbf{x}_{\mathcal{E}_\mathcal{A}}\;\perp\!\!\!\perp\;\mathbf{x}_{\mathcal{E}_\mathcal{B}} \ . \]
        
    \end{enumerate}
    
\end{theorem}
\begin{proof}

 1) Let $G' = (\mathcal{E},\mathcal{L}')$ be the GMRF representation of $\mathbf{x}_\mathcal{E}$, whose node set is $\mathcal{E}$ and links are given by
\[
 \{e_i,e_j\} \in \mathcal{L}'
 \;\Longleftrightarrow\;
 (\bm{\Omega}_\mathcal{E})_{ij} \neq 0, \qquad i \neq j.
\]
By \Cref{eq:precision_E_paper}, each non-zero entry of $\bm{\Omega}_\mathcal{E}$ arises from a non-zero entry of either $\mathbf B_1^\top \operatorname{diag}(\mathbf{d}_\mathcal{V}) \mathbf B_1$ or $\mathbf B_2 \operatorname{diag}(\mathbf{d}_\mathcal{T}) \mathbf B_2^\top$ (possibly both). Hence, by construction, the standard GMRF graph $G'$ is a subgraph of the CMRF, i.e., $G' \subseteq G$.  It is straightforward that graph inclusion preserves the global Markov property. Since it holds for $G'$, it therefore holds 
for $G$ as well. \\
2) We define the \emph{lower} and \emph{upper} precision matrices as
\begin{align}
  \bm{\Omega}_{ d}
  &:= k\,\mathbf{I}_{|\mathcal{E}|} 
      - \mathbf{B}_1^\top \operatorname{diag}(\mathbf{d}_\mathcal{V}) \mathbf{B}_1, \label{eq:lower_prec}\\
  \bm{\Omega}_{u}
  &:= k\,\mathbf{I}_{|\mathcal{E}|} 
      - \mathbf{B}_2 \operatorname{diag}(\mathbf{d}_\mathcal{T}) \mathbf{B}_2^\top. \label{eq:upper_prec}
\end{align}
From \Cref{eq:lower_prec,eq:upper_prec}, it follows immediately that
\begin{equation}\label{eq:omega_dec}
    \bm{\Omega}_\mathcal{E}= \bm{\Omega}_{ d} + \bm{\Omega}_{u} - k\mathbf{I}_{|\mathcal{E}|} 
\end{equation}
Moreover, using $\mathbf{B}_1 \mathbf{B}_2 = \mathbf{0}$, it is easy to verify that
\begin{equation}\label{eq:omega_factor}
  \bm{\Omega}_\mathcal{E} = k^{-1} \bm{\Omega}_{ d} \bm{\Omega}_{ u}
              = k^{-1} \bm{\Omega}_{u} \bm{\Omega}_{d}.
\end{equation}
Note that, since $\bm{\Omega}_\mathcal{E}\succ0$, $\bm{\Omega}_{d}$ and $\bm{\Omega}_{ u}$ are symmetric positive definite, hence invertible.
It can be proved that the covariance matrix $\bm{\Sigma}_\mathcal{E} = \bm{\Omega}_\mathcal{E}^{-1}$ can be written as
\begin{equation}\label{eq:cov_dec}
\bm{\Sigma}_\mathcal{E} = \bm{\Omega}_{u}^{-1} + \bm{\Omega}_{d}^{-1} - k^{-1} \mathbf{I}_{|\mathcal{E}|},
\end{equation}
by verifying the identity $\bm{\Sigma}_\mathcal{E} \bm{\Omega}_\mathcal{E} = \mathbf{I}_{|\mathcal{E}|}$ using \Cref{eq:omega_dec,eq:omega_factor}.

Let $\mathcal{E}_\mathcal{A},\mathcal{E}_\mathcal{B} \subseteq \mathcal{E}$ be a pair of sets
in $G$ such that $\mathcal{E}_\mathcal{A} \perp_{\mathrm{col}} \mathcal{E}_\mathcal{B}$. By Gaussianity of 
$\mathbf{x}_\mathcal{E}$, to prove 2) it is sufficient to show that
\[
\mathrm{Cov}(x_{e_i},x_{e_j}) = 0 
\quad \text{for every } e_i \in \mathcal{E}_\mathcal{A},\ e_j \in \mathcal{E}_\mathcal{B}.
\]
\begin{figure*}
\centering
\scalebox{0.9}{
\parbox{\textwidth}{
\hrule
\vspace{0.3em}

\begin{equation}\label{eq:edge_wise_h}
    \Phi_{h,e}[t] = \mathbb{E}\Big\{\frac{k}{2}\,(y_e[t] - \mathbf{u}_e^\top[t]\bm\theta)^2\Big\}
\end{equation}
\begin{equation}\label{eq:edge_wise_l}
 \Phi_{d,e}[t] = \mathbb{E}\Big\{\frac{1}{2} \Big[ 
        (\mathbf{d}_{\mathcal{V},u_1} + \mathbf{d}_{\mathcal{V},u_2}) (y_e[t] - \mathbf{u}_e^\top[t]\bm\theta)^2
        + \sum_{u \in\{u_1,u_2\}} \sum_{e' \neq e: u \trianglelefteq e'} \mathbf{d}_{\mathcal{V},u_1} \mathbf{B}_1[e,u]\mathbf{B}_1[e',u] (y_e[t] - \mathbf{u}_e^\top[t]\bm\theta)(y_{e'}[t] - \mathbf{u}_{e'}^\top[t]\bm\theta) \Big\}
\end{equation}
\begin{equation}\label{eq:edge_wise_u}
\Phi_{u,e}[t]= \mathbb{E}\Big\{\frac{1}{2} \Big[
        \sum_{t\in \mathcal{T}: e\in t} \Big(
            \mathbf{d}_{\mathcal{T},t} (y_e[t] - \mathbf{u}_e^\top[t]\bm\theta)^2 
            + \sum_{e'\neq e: e'\in t} \mathbf{d}_{\mathcal{T},t} \mathbf{B}_2[t,e]\mathbf{B}_2[t,e'] (y_e[t] - \mathbf{u}_e^\top[t]\bm\theta)(y_{e'}[t] - \mathbf{u}_{e'}^\top[t]\bm\theta)\Big)\Big]\Big\}
\end{equation}

\vspace{0.1em}
\hrule
}}
\end{figure*}
Fix $e_i \in \mathcal{E}_\mathcal{A}$ and $e_j \in \mathcal{E}_\mathcal{B}$. Since 
$\mathcal{E}_\mathcal{A} \perp_{\mathrm{col}} \mathcal{E}_\mathcal{B}$, $e_i$ and $e_j$ belong to two distinct
connected components of the lower graph $G_{d}=(\mathcal{E},\mathcal{L}_d),$ with $\mathcal{L}_d$ defined as in \eqref{eq:lower_inter}. By construction, the sparsity pattern of $\bm{\Omega}_{d}$ encodes the connectivity of $G_{d}$. Thus, after a reordering of the edge indices, the matrix $\bm{\Omega}_{d}$
takes a two block-diagonal form, with the blocks corresponding to the components of $G_{d}$ containing $e_i$ and $e_j$, respectively. Hence, its inverse $\bm{\Omega}_{d}^{-1}$ has the same
block structure, giving, in particular, $(\bm{\Omega}_{d}^{-1})_{ij} = 0.$
The same argument applied to $G_{u}=(\mathcal{E},\mathcal{L}_u),$ with $\mathcal{L}_u$ defined as in \Cref{eq:upper_inter}, yields $(\bm{\Omega}_{u}^{-1})_{ij} = 0.$ 

Using the decomposition in \Cref{eq:cov_dec}, we obtain, for $i \neq j$,
\begin{equation}
\mathrm{Cov}(x_{e_i},x_{e_j})
= \mathbf{\Sigma}_{ij}
= (\bm{\Omega}_{d}^{-1})_{ij} + (\bm{\Omega}_{u}^{-1})_{ij}
= 0. \qedhere   
\end{equation}
 \end{proof}
Note that the CMRF representation of edge signals retains the expressiveness of a standard MRF in encoding conditional dependencies among variables via the Markov property. In fact, the connectivity of the graph describes the paths along which information can propagate. However, by color separation, the CMRF is able to reveal marginal independences that arise from the different nature of interactions. This heterogeneity depends on the higher-order domain on which variables originate.

Revisiting the CMRF in~\Cref{fig:top-and-cmrf}, according to the coloring of links, $x_{e_1}$ is independent of $(x_{e_4},x_{e_5},x_{e_6},x_{e_7})$, since there is no monochromatic path from $e_1$ to  $\{e_4,e_5,e_6,e_7\}$ in the CMRF. 
This independence cannot be deduced from the mere graph connectivity of the corresponding uncolored GMRF.

\section{Diffusion Adaptation over Colored Markov Random Fields}

We illustrate the utility of our framework through a distributed estimation case study by extending earlier protocols operating among agents collecting node measurements \cite{di2014diffusion}. 
As motivated in the introduction, we instead consider sensors placed on the edges of physical networks to measure flows between 
adjacent nodes, and we aim at quantifying the gain coming from higher-order topological information.

For each edge $e$ and time instant $t$, we consider
the following linear measurement model of the parameter $\bm\theta_0 \in \reall^M$:
\begin{equation}
    y_e[t] = \mathbf{u}_e^\top[t]\bm\theta_0 + n_e[t] \ ,
\end{equation}
where $\mathbf{u}_e^\top[t]$ is a known time-varying $M$ dimensional regressor, independent over time and space.
Crucially, the noise $\mathbf{n}_\mathcal{E}$ spatially correlates the measurements, being a Gaussian edge signal $\mathbf{n}_\mathcal{E} \sim \mathcal{N}(0, \bm{\Omega}_\mathcal{E}^{-1})$, with $\bm{\Omega}_\mathcal{E}$ as in \Cref{eq:precision_E_paper}, parametrized by a scalar $k>0$ and the non-negative coefficient vectors $\mathbf{d}_\mathcal{V},\mathbf{d}_\mathcal{T}$.
Thus, the CMRF of interest represents the \textit{logic network} layer, capturing statistical dependencies among measurements. 
By aggregating at each time-step the edge measurements in the vector $\mathbf{y}[t]=[y_{e_1}[t],...,y_{e_{|\mathcal{E}|}}[t]]^\top$ and the regressors in the matrix $\mathbf{U}[t]=[\mathbf{u}_{e_1},...,\mathbf{u}_{e_{|\mathcal{E}|}}]^\top$, we cast a maximum likelihood estimation problem. 
Assuming the model is stationary, and denoting by $\| \cdot\|_{\bm{\Omega}_\mathcal{E}}$ the norm weighted by $\bm{\Omega}_\mathcal{E}$, we consider
\begin{equation}\label{eq:global_problem}
    \bm\theta_0 = \underset{\bm\theta}{\mathrm{argmin}}\ \mathbb{E}\{\|\mathbf{y}[t]-\mathbf{U}[t]\bm\theta\|_{\bm{\Omega}_\mathcal{E}}^2\} \ .
\end{equation}
We solve \Cref{eq:global_problem} using instantaneous information and local signaling among communicating agents \cite{chen2012diffusion}.
In detail, we leverage adapt-then-combine (ATC) schemes based on local stochastic gradient descent (\textit{adaptation}) and averaging among neighbors on the line graph (\textit{combination}).
The latter corresponds to the graph encoding the lower adjacencies between edges in the simplicial complex (cf. \Cref{sec:background}). 
Notably, since in a simplicial complex the upper adjacency implies the lower adjacency, this \textit{communication network} layer is a sufficient infrastructure to support a complete exchange of information.

\begin{table}[t]
\centering
\caption{Summary of experimental setup.}
\begin{tabular}{ll}
\toprule
\
\textbf{Parameter} & \textbf{Specification / Distribution} \\
\midrule\midrule
2-SC & $|\mathcal{V}| = 10, |\mathcal{E}| = 21, |\mathcal{T}| = 12$ \\ \midrule
$\mathbf{d}_\mathcal{V}, \mathbf{d}_\mathcal{T} $& $\mathrm{Unif}(0.2, 5)$ \\ \midrule
$k$& $\lambda_{\max}\{\mathbf{B}_1^\top \mathrm{diag}(\mathbf{d}_\mathcal{V}) \mathbf{B}_1 + \mathbf{B}_2 \mathrm{diag}(\mathbf{d}_\mathcal{T}) \mathbf{B}_2^\top\} + 0.1$ \\ \midrule
$\theta_0$ & $\mathcal{N}(0, \mathbf{I}_{10})$ \\ \midrule
$\mathbf{u}_e[k]$& $\mathcal{N}(0, 0.2 \cdot \mathbf{I}_{10})$ \\ \midrule
$\mu$& $5\cdot10^{-3}$ for ATC-CMRF  \\
(learning rate)& Others are scaled to match convergence rate \\
\bottomrule
\end{tabular}
\label{tab:simplicial_setup}
\end{table}

By defining $\mathbf{D}_\mathcal{V}:=\operatorname{diag}(\mathbf{d}_\mathcal{V})$ and $\mathbf{D}_\mathcal{T}:=\operatorname{diag}(\mathbf{d}_\mathcal{T})$, Prob. \eqref{eq:global_problem} decouples into three terms accordingly to \Cref{eq:precision_E_paper}:
\begin{equation} \begin{aligned} & \mathbb{E}\big\{\|\mathbf{y}[t]-\mathbf{U}[t]\bm\theta\|_{\bm{\Omega}_\mathcal{E}}^2\big\} = \mathbb{E}\bigg\{\underbrace{\frac{k}{2}\|\mathbf{y}[t]-\mathbf{U}[t]\bm\theta\|_{2}^2}_{\ell_{h}[t]} \bigg\} + \\ 
-& \mathbb{E}\bigg\{\!\underbrace{\frac{1}{2}\|\mathbf{B}_1(\mathbf{y}[t]\!-\!\mathbf{U}[t]\bm\theta)\|_{\mathbf{D}_{\mathcal{V}}}^2}_{\ell_d[t]}\!\bigg\} \!-\!\mathbb{E}\bigg\{\!\underbrace{\frac{1}{2}\|\mathbf{B}_2^\top(\mathbf{y}[t]\!-\!\mathbf{U}[t]\bm\theta)\|_{\mathbf{D}_\mathcal{T}}^2}_{\ell_u[t]}\!\bigg\}. \notag \end{aligned} \end{equation}

These three terms in turn can be decomposed edge-wise as 
\begin{equation}
    \ell_h[t]-\ell_d[t]-\ell_u[t]=\sum_{e \in \mathcal{E}} \Phi_e[t] = \sum_{e \in \mathcal{E}}\{\Phi_{h,e}[t]-\Phi_{d,e}[t]-\Phi_{u,e}[t]\},
\end{equation}
as specified in \Cref{eq:edge_wise_h,eq:edge_wise_l,eq:edge_wise_u}, enabling collective estimation at every topological level by the aforementioned ATC scheme.
We do not delve into details about convergence properties of the algorithm, since it follows similar arguments as in \cite{di2014diffusion}.

We validate the proposed procedure via numerical simulation over a 2-simplicial complex (2-SC) built on top of a random ER graph: \Cref{tab:simplicial_setup} summarizes the setup.
\begin{figure}[!t]
    \centering
    \includegraphics[width = 1\linewidth]{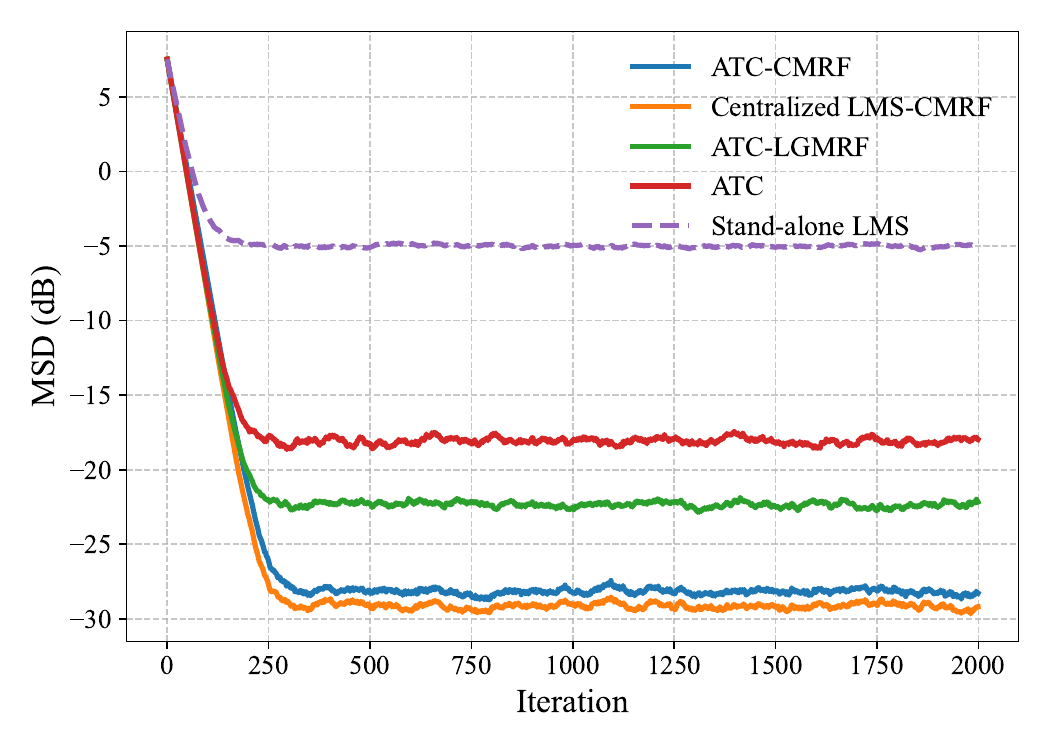}
\caption{Average MSD for the proposed simulations} 
    \label{fig:cmrf}
\end{figure}
In this synthetic scenario, we run 100 independent simulations to retrieve a Monte Carlo estimate of the mean squared deviation $\mathrm{MSD} = \lim_{t \rightarrow\infty} \frac{1}{|\mathcal{E}|}\sum_{e \in \mathcal{E}}\|\bm\theta_e[t]-\bm\theta_0\|_2^2$ over 2000 iterations. 
We compare ATC--CMRF with four baselines: (i) stand-alone LMS, where each agents performs purely local adaptation; (ii) centralized LMS--CMRF, a fully informed centralized method; and two ATC variants—(iii) ATC--LGMRF, which exploits only the lower-graph MRF ($\Phi_e[t] = \Phi_{h,e}[t] - \Phi_{d,e}[t]$); and (iv) ATC, which ignores spatial correlation ($\Phi_e[t] = \Phi_{h,e}[t]$).

ATC-CMRF nearly matches centralized LMS-CMRF, while the other distributed protocols yield only partial gains over stand-alone LMS, consistently with their reduced topological information.
Conversely, ATC-CMRF can subsume these variants: setting $\mathbf{d}_{\mathcal T}=0$ recovers ATC-LGMRF, and setting $\mathbf{d}_{\mathcal V}=\mathbf{d}_{\mathcal T}=0$ recovers ATC, highlighting its superior performance and greater expressivity.

\section{Conclusions}
We introduced Colored Markov Random Fields, graphical models that encode
statistical independencies of Gaussian edge signals over simplicial complexes. CMRFs
augment GMRFs with link coloring, preserving the Markov property while yielding additional marginal independence statements for edge variables. We evaluated
the theoretical framework on a distributed estimation problem with sensors on network links, and
numerical experiments on synthetic data showed that the resulting topology-aware
representation outperforms alternatives with absent or partial topological information.

Future work includes extending CMRFs beyond the Gaussian setting (e.g., to discrete or
heavy-tailed edge signals) and developing fully data-driven schemes that, unlike the
current case study with prescribed noise structure, jointly learn the topology-aware
noise model (the CMRF parameters) and perform distributed estimation over network links.

\bibliographystyle{unsrt}
\bibliography{bibliography}

\end{document}